\pdfoutput=1

\documentclass[11pt]{article}

\usepackage{naacl}
\usepackage{custom}

\usepackage{amsmath,amsfonts,bm}
\usepackage{scalerel}

\newcommand{\paroutline}[3][false]{%
    \ifnum\pdfstrcmp{#1}{true}=0
        #3
    \else
        [\textit{\textcolor{DiverseMagenta}{#2}}] \textcolor{AccentBlue}{#3}
    \fi
}

\newcommand{\ifcondition}{\textbf{ if }}
\newcommand{\otherwisecondition}{\textbf{ otherwise }}

\newcommand{\lmone}{\mymacro{p}}
\newcommand{\lmtwo}{\mymacro{q}}

\DeclareTextFontCommand{\breaktexttt}{\ttfamily\hyphenchar\font=45\relax}

\newcommand{\vtheta}{\mymacro{\boldsymbol{\theta}}}

\newcommand{\Count}{\mymacro{\#}}
\newcommand{\COfC}{\mymacro{r}}
\newcommand{\CountFun}[1]{\mymacro{\Count\mleft(#1\mright)}}
\newcommand{\typecount}{\mymacro{\Count^n_{\textsc{t}}}}
\newcommand{\CountGT}{\mymacro{\Count_{\mathrm{GT}}}}
\newcommand{\CountKZ}{\mymacro{\Count_{\mathrm{K}}}}

\newcommand{\pGT}{\mymacro{\smoothed{q}^n_{\mathrm{GT}}}}
\newcommand{\pJM}{\mymacro{\smoothed{q}^n_{\mathrm{JM}}}}
\newcommand{\pJMminus}{\mymacro{\smoothed{q}^{n-1}_{\mathrm{JM}}}}
\newcommand{\pK}{\mymacro{\smoothed{q}^n_{\mathrm{K}}}}
\newcommand{\pKN}{\mymacro{\smoothed{q}^n_{\mathrm{KEN}}}}
\newcommand{\pKNunigram}{\mymacro{\smoothed{q}^1_{\mathrm{KEN}}}}
\newcommand{\pKNminus}{\mymacro{\smoothed{q}^{n-1}_{\mathrm{KEN}}}}

\newcommand{\lambdapos}{\textcolor{ETHBlue}{Z_{+}}}
\newcommand{\lambdaneg}{\textcolor{ETHBlue}{Z_{-}}}

\newcommand{\ind}[1]{\mathbbm{1} \left\{ #1 \right\}}

\newcommand{\bos}{\mymacro{\textsc{bos}}}
\newcommand{\eos}{\mymacro{\textsc{eos}}}

\newcommand{\alphabet}{\mymacro{\Sigma}}
\newcommand{\alphabeteos}{\mymacro{\overline{\Sigma}}}
\newcommand{\kleene}[1]{\mymacro{#1^*}}
\newcommand{\defn}[1]{\mymacro{\textbf{#1}}}
\newcommand{\defeq}{\mathrel{\stackrel{\textnormal{\tiny def}}{=}}}

\newcommand{\params}{\mymacro{\boldsymbol{\theta}}}

\newcommand{\qtheta}{\mymacro{q_{\vtheta}}}
\newcommand{\data}{\mymacro{\mathcal{D}}}
\newcommand{\pemp}{\mymacro{p_{\data}}}
\newcommand{\pngramemp}{\mymacro{p_{\data}^n}}
\newcommand{\smoothpngramemp}{\mymacro{\smoothed{p}_{\data}^n}}
\newcommand{\smoothed}[1]{\mymacro{\tilde{#1}}}
\newcommand{\qsmooth}{\mymacro{\smoothed{q}}}
\newcommand{\qsmoothTheta}{\mymacro{\qsmooth_\params}}

\newcommand{\ppos}{\mymacro{p_{+}}}
\newcommand{\pneg}{\mymacro{p_{-}}}
\newcommand{\qthetasmooth}{\mymacro{\smoothed{q}_\params}}
\newcommand{\qngram}{\mymacro{q_{\mathrm{MLE}}^n}}
\newcommand{\qngramsmooth}{\mymacro{\smoothed{q}_{\mathrm{MLE}}^n}}

\newcommand{\ppre}{\mymacro{\pi}}
\newcommand{\ppreT}{\mymacro{\ppre_T}}
\newcommand{\p}{\mymacro{p}}
\newcommand{\word}{\mymacro{x}}
\newcommand{\wordtwo}{\mymacro{y}}
\newcommand{\regPar}{\mymacro{\gamma}}
\newcommand{\regParLS}{\mymacro{\gamma_{\mathrm{LS}}}}
\newcommand{\regParSpace}{\mymacro{\Gamma}}

\newcommand{\regParPos}{\mymacro{\regPar_{+}}}
\newcommand{\regParNeg}{\mymacro{\regPar_{-}}}

\newcommand{\prefixemp}{\mymacro{\pi}_{\data}}

\newcommand{\str}{\mymacro{\boldsymbol{x}}}
\newcommand{\strtwo}{\mymacro{\boldsymbol{y}}}
\newcommand{\nstr}{\mymacro{\str^{n}}}

\newcommand{\nstrt}{\mymacro{\str^{n}_t}}

\newcommand{\nmstr}{\mymacro{\str^{n-1}}}

\newcommand{\history}{\mymacro{\nstr}}
\newcommand{\historyt}{\mymacro{\nstrt}}

\newcommand{\regularizer}{\mymacro{\mathcal{R}}}
\newcommand{\regularizerls}{\mymacro{\regularizer_{\mathrm{LS}}}}

\newcommand{\sampleidx}{\mymacro{m}}
\newcommand{\numsamples}{\mymacro{M}}

\newcommand{\set}[1]{\left\{ #1 \right\}}
\newcommand{\eosalphabet}{\mymacro{\overline{\alphabet}}}
\newcommand{\bosstringsn}{\mymacro{\alphabet^{n-1}_\bos}}
\newcommand{\bosstrings}{\mymacro{\alphabet^\ast_\bos}}

\def\1{\mathbf{1}}
\newcommand{\dataset}{{\mymacro{\mathcal{D}}}}

\def\ru{{{\mymacro{u}}}}

\def\vtheta{{{\mymacro{ \boldsymbol{\theta}}}}}

\newcommand{\KL}{{\mymacro{D_{\mathrm{KL}}}}}
\newcommand{\entropy}{{\mymacro{\mathrm{H}}}}

\newcommand{\bigO}[1]{{\mymacro{ \mathcal{O}\left(#1\right)}}}

\newcommand{\annot}[1]{\textcolor{black!50}{\scaleto{\quad\text{#1}}{8pt}}}

\newcommand{\dd}[1]{\mymacro{\scaleto{\,\pm#1}{6pt}}}
\newcommand{\pars}[1]{\mymacro{\scaleto{#1}{8pt}}}
\newcommand{\textword}[1]{\mymacro{\textit{#1}}}

\title{The Role of $n$-gram Smoothing in the Age of Neural Networks}

\author{Luca Malagutti~\;~\;~Andrius Buinovskij~\;~\;~Anej Svete\\
\bf{Clara Meister}~\;~\;~\bf{Afra Amini}~\;~\;~\bf{Ryan Cotterell} \\
  \texttt{\href{mailto:lmalagutti@inf.ethz.ch}{lmalagutti@inf.ethz.ch}}~\;~\texttt{\href{mailto:andriusb@student.ethz.ch}{andriusb@student.ethz.ch}}\\
    \texttt{\{\href{mailto:asvete@inf.ethz.ch}{asvete}, \href{mailto:meistecl@inf.ethz.ch}{meistecl}, \href{mailto:afra.amini@inf.ethz.ch}{aamini}, \href{ryan.cotterell@inf.ethz.ch}{ryan.cotterell}\}@inf.ethz.ch} \\
   {%
\setlength{\fboxsep}{2.5pt}%
\setlength{\fboxrule}{2.5pt}%
\fcolorbox{white}{white}{
    \includegraphics[width=.15\linewidth]{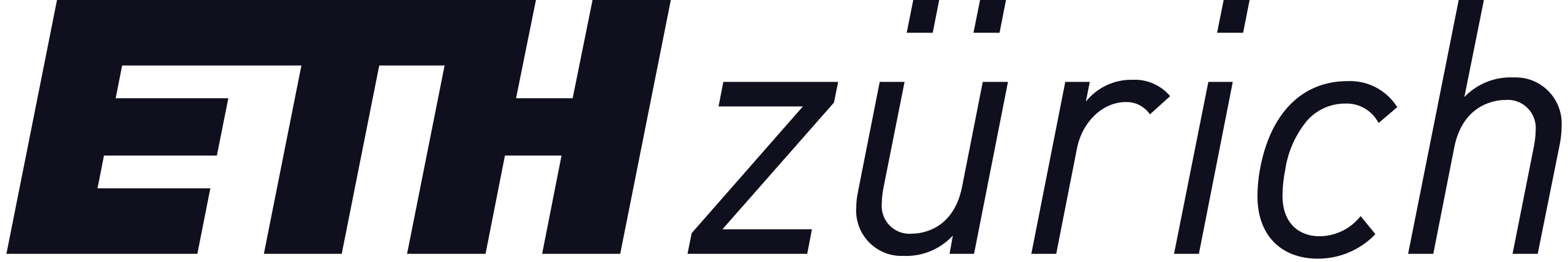}
}
}
}

\begin{document}
\maketitle

\begin{abstract}
    For nearly three decades, language models derived from the $n$-gram assumption held the state of the art on the task.
    The key to their success lay in the application of various smoothing techniques that served to combat overfitting.
    However, when neural language models toppled $n$-gram models as the best performers, $n$-gram smoothing techniques became less relevant.
    Indeed, it would hardly be an understatement to suggest that the line of inquiry into $n$-gram smoothing techniques became dormant.
    This paper re-opens the role classical $n$-gram smoothing techniques may play in the age of neural language models.
    First, we draw a formal equivalence between label smoothing, a popular regularization technique for neural language models, and add-$\lambda$ smoothing.
    Second, we derive a generalized framework for converting \emph{any} $n$-gram smoothing technique into a regularizer compatible with neural language models.
    Our empirical results find that our novel regularizers are comparable to and, indeed, sometimes outperform label smoothing on language modeling and machine translation.
    
\vspace{0.5em}
\hspace{.5em}\includegraphics[width=1.25em,height=1.15em]{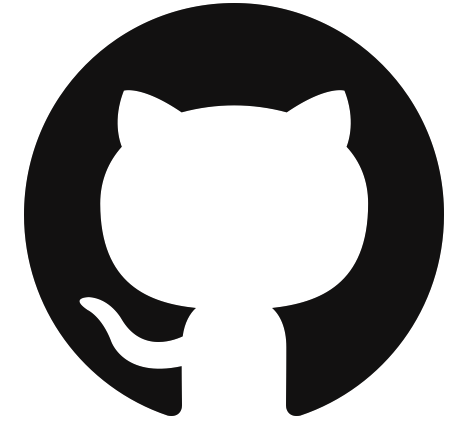}\hspace{.75em}
\parbox{\dimexpr\linewidth-7\fboxsep-7\fboxrule}{\url{https://github.com/rycolab/ngram_regularizers}}
\vspace{-.5em}
\end{abstract}

\section{Introduction} \label{sec:intro}
Let $\alphabet$ be an \defn{alphabet}.\footnote{An alphabet is a finite, non-empty set.}
A \defn{language model} is a probability distribution $\p$ over $\kleene{\alphabet}$, the set of all strings $\str = x_1 \cdots x_T $ with symbols $x_t$ drawn from $\alphabet$. 
A fundamental task in natural language processing (NLP) is to estimate a language model---often from a parametric family---that places a high probability on held-out, human-generated text. 
A common design choice is to construct a locally normalized language model, i.e., one which factorizes autoregressively\footnote{Autoregressivization is without loss of generality \citep[Theorem 2.4.2]{cotterell2023formal}.} as $\p(\str) = \p(\eos \mid \str) \prod_{t=1}^{|\str|} \p(\word_t \mid\str_{<t})$, where $\eos \not\in\alphabet$ is a distinguished \underline{e}nd-\underline{o}f-\underline{s}tring symbol and $\str_{<t} \defeq x_1 \cdots x_{t-1}$ is a prefix of $\str$.\looseness=-1

For years, the best parametric families of language models for this estimation task applied the $n$-gram assumption, detailed below.
\begin{figure}
    \centering
    \begin{tikzpicture}
        \node[align=center] (data) at (0, 1) {$\dataset$};
        \node[align=center] (pData) at (0, 0) {$\pemp$};
        \node[align=center] (qtheta) at (3, 0) {$\qtheta$};
        \node[align=center] (pDataSmooth) at (0, -2.5) {$\smoothpngramemp$};
        \node[align=center] (qsmoothTheta) at (3, -2.5) {$\qsmoothTheta$};
        
        \node[align=center] at (4.5, 0 + 0.15) {\footnotesize Language};
        \node[align=center] at (4.5, 0 - 0.15) {\footnotesize model};
        
        \node[align=center] at (4.5, -2.5 + 0.4) {\footnotesize Smoothed or};
        \node[align=center] at (4.5, -2.5) {\footnotesize regularized};
        \node[align=center] at (4.5, -2.5 - 0.4) {\footnotesize language model};
        
        \node[align=center] at (-1.25, -1.25 + 0.4) {\footnotesize Smoothing};
        \node[align=center] at (-1.25, -1.25) {\footnotesize technique};
        \node[align=center] at (-1.25, -1.25 - 0.4) {\footnotesize (e.g., add-$\lambda$)};
        
        \draw[->] (data.south) -- (pData.north);
        \draw[-Latex, thick, dashed] (pData.east) -- (qtheta.west) node[midway, sloped, above] {\footnotesize $\min \KL$};
        \draw[->] (pData.south) -- (pDataSmooth.north);
        \draw[-Latex, thick, dashed] (pDataSmooth.east) -- (qsmoothTheta.west) node[midway, sloped, above] {\footnotesize $\min \KL$};
        \draw[-Latex, thick, dashed, ETHBlue] (pData) -- (qsmoothTheta) node[midway, sloped, above] {\footnotesize $\min \KL + \regularizer$};
    \end{tikzpicture}
    \caption{An illustration of the introduced framework.
    With maximum-likelihood estimation (MLE), a language model $\qtheta$ is trained to match $\pemp$, the empirical distribution induced by a dataset $\dataset$.
    However, we can also modify (\emph{smooth}) $\pemp$ into $\smoothpngramemp$ and train a language model $\qthetasmooth$ on $\smoothpngramemp$.
    We show that the latter can be thought of as training $\qthetasmooth$ with a regularized maximum-likelihood objective.\looseness=-1}
    \label{fig:figure-1}
    \vspace{-10pt}
\end{figure}
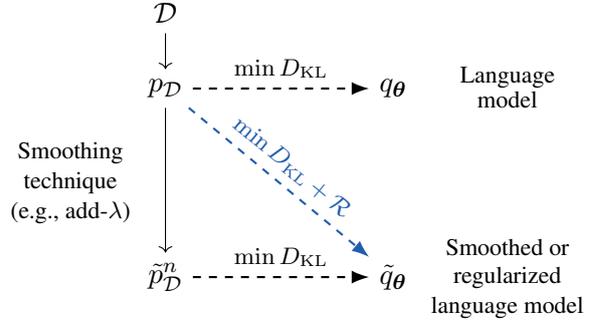 
\begin{assumption}[$n$-gram assumption]\label{ass:ngram}
    A language model obeys the \defn{$n$-gram assumption} if the following conditional independence holds
    \begin{equation}
        \begin{aligned}
            \p(\word_t \mid\str_{<t})  & \defeq \p(\word_t \mid  \word_1 \cdots \word_{t-1})                          \\
                           & =  \p(\word_t \mid \word_{t-n+1} \cdots \word_{t-1}) \\
                           & \defeq \p\left(\word_t \mid \historyt\right)
        \end{aligned}
    \end{equation}
    where $\word_t$ for $t < 1$ is treated as a distinguished padding symbol $\bos \not\in \alphabet$.\footnote{The symbol $\bos$ stands for \underline{b}eginning \underline{o}f \underline{s}tring.}
    We will call $\historyt \defeq \word_{t-n+1} \cdots \word_{t-1}$ the \defn{history} of $\word_t$.
    We use $\history$ for histories where the time step $t$ is irrelevant.
\end{assumption}
\noindent \Cref{ass:ngram} was, historically, considered a practically effective manner to fight the curse of dimensionality, despite its inability to attend to contexts longer than $n-1$ words and thus capture long-range dependencies.\footnote{$n$-gram LMs are linguistically very primitive. They are an instance of \emph{strictly local} languages, one of the simplest language classes \citep{Jager2012-kv}. It has long been argued that we require LMs that are more expressive and capture more complex phenomena in human language \citep{Chomsky57a,chelba-jelinek-1998-exploiting-syntactic,abney1999}.}

The maximum-likelihood estimator of a model $\qngram$ under the $n$-gram assumption is straightforward to derive.\footnote{We discuss the MLE of language models in detail in \cref{sec:preliminaries}.}
Indeed, one can express it simply using counts of sub-string occurrences:
\begin{equation}
    \qngram\left(\word \mid \history\right) = \frac{\CountFun{\history \word}}{\CountFun{\history}} \label{eq:qngram}
\end{equation}
where, intuitively, $\CountFun{\strtwo}$ denotes the number of times substring $\strtwo$ occurs in the training dataset.\footnote{We formally introduce the counting function $\Count$ in \cref{sec:preliminaries}  accounting for the presence of $\eos$ at the end of each string.} 
Critically, the simplicity enforced by \Cref{ass:ngram} alone is \emph{not} enough to prevent overfitting for reasonably sized $n$: a minimally parameterized $n$-gram model has $\bigO{|\alphabet|^n}$ free parameters, one for each $n$-gram.
Therefore, the maximum-likelihood solution of an $n$-gram language model overfits on its training data by assigning probability $0$ to any string containing an $n$-gram that does not occur in the training dataset, which is undesirable.
To solve this issue, in addition to making the $n$-gram assumption, modelers applied a variety of \defn{smoothing} techniques to regularize the estimation of $n$-gram probabilities and obtain smoothed probabilities $\qngramsmooth$.

One of the simplest $n$-gram smoothing techniques is known as \defn{add-$\lambda$ smoothing},\footnote{In the special case that $\lambda =1$, this technique is generally referred to as Lidstone smoothing.} which can be informally described as hallucinating $n$-grams in the training dataset that occur at frequency $\lambda$.
In other words, the counts of all $n$-grams---including those that are unobserved in the dataset---are augmented by $\lambda$:
\begin{equation} \label{eq:add-lambda-ml}
    \qngramsmooth\left(\word \mid \history\right) \defeq \frac{\CountFun{\history \word} + \lambda}{\CountFun{\history} + \left(|\alphabet| + 1\right) \lambda}.
\end{equation}
In the context of NLP, add-$\lambda$ smoothing has often served a pedagogical purpose and is thus commonly taught, but its efficacy in practice was considered limited \citep{eisner2023hw}. 
However, researchers developed more sophisticated related smoothing methods \cite{jelinek,katz,kneser} that were useful in practice.
\citet{kneser}, for example, was long considered the best available method \citep{chen_empirical} and its popularity inspired a principled Bayesian interpretation \citep{teh2006bayesian}. 
Interestingly though, add-$\lambda$ smoothing still has a place in today's literature; as we prove in \Cref{background:label-smoothing}, it is identical to a regularization technique called label smoothing, common in the training of (conditional) neural language models \citep{pereyra2017regularizing,meister2020generalized}, and particularly common in machine translation~\citep{nllbteam2022language}.\looseness=-1

In this context, we present the primary theoretical research question addressed in our paper.
If add-$\lambda$ smoothing is equivalent to label smoothing---an effective regularization technique for training today's neural language models---can we reverse-engineer further regularization methods starting from other $n$-gram smoothing methods?
Given that add-$\lambda$ smoothing was not known to perform well compared to other smoothing techniques in the context of $n$-gram language models, it is natural to suspect that reverse-engineered regularizers based on empirically more successful smoothing techniques may indeed outperform label smoothing.

To this end, we derive a straightforward relationship between training on a smoothed empirical distribution and additive regularization methods, e.g., entropy regularization. 
Using that relationship, we further show that \emph{any} smoothing method can be reformulated as an additive regularization of the standard maximum-likelihood objective, and, hence, we provide an explicit way to construct the regularizer that can be applied to the training of any (neural) language model, helping connect classical smoothing methods and modern language models by introducing a way to incorporate smoothing techniques into the training of neural language models.
We complement our theoretical analysis by empirically verifying the validity of our proposed methods on two small-scale datasets and observe that some regularizers based on more complex $n$-gram smoothing techniques do indeed perform better than label smoothing on language modeling and machine translation.\looseness=-1

\section{Label Smoothing and add-$\lambda$ Smoothing}\label{background:label-smoothing}
In this section, we derive a formal relationship between the add-$\lambda$ smoothing applied to $n$-gram models and label smoothing.
The relationship is based on a regularizer that, when applied during MLE, simulates add-$\lambda$ smoothing exactly.
Additionally, we contend that, in this sense, label smoothing \emph{generalizes} add-$\lambda$ smoothing from a method applicable only to $n$-gram language models to one that can be used with neural language models.\footnote{There is no obvious manner to apply add-$\lambda$ smoothing to the training of a neural language model since it is designed as an operation on count-based data.}
To the best of the authors' knowledge, this derivation and the relationship it exposes are novel.
We re-use this meta paradigm in \Cref{sec:new-regularizers} to derive novel regularizers, which we then compare experimentally to label smoothing.\looseness=-1

\subsection{Preliminaries} \label{sec:preliminaries}

\paragraph{Some Notation.}
We define $\eosalphabet \defeq \alphabet \cup \set{\eos}$.
To make it possible to always condition on exactly $n-1$ symbols, whenever the history is shorter than $n-1$ symbols, we prepend a string with the appropriate number of $\bos$ symbols. 
Thus, we define $\bosstringsn \defeq \bigcup_{\ell = 1}^{n - 1} \set{\bos}^\ell \times \alphabet^{n - 1 - \ell}$ (the set of all possible histories) and $\bosstrings \defeq \bosstringsn \cup \bigcup_{\ell = n}^\infty \alphabet^\ell$ as the set of all strings that are either prefixed by $\bos$ or contain more than $n - 1$ symbols.

\paragraph{Maximum-likelihood Estimation.}
We now introduce maximum-likelihood estimation in the context of language modeling.
Suppose we observe a collection of samples $\data = \{\str^{(\sampleidx)}\}_{\sampleidx=1}^\numsamples$ where $\str^{(\sampleidx)} \sim \p$ and $\p$ is the distribution over strings that we are trying to model.
Let $\pemp$ be the empirical distribution induced by $\data$, i.e., the probability distribution defined as
\begin{equation}
    \pemp(\str) = \frac{1}{\numsamples}\sum_{\sampleidx=1}^\numsamples \mathbbm{1}\{ \str = \str^{(\sampleidx)}\}.
    \label{eq:empirical}
\end{equation}
Choosing a model $\qtheta$ that minimizes the forward KL divergence $\KL(\pemp \mid\mid \qtheta)$ is known as maximum-likelihood estimation.
Under regularity conditions~\citep{Cam1953OnSA}, MLE is consistent, i.e., in the limit of infinite data, we arrive at the true parameters of the data-generating distribution if the data-generating distribution indeed came from the model's parametric family.\looseness=-1

\paragraph{Counting Substrings in a Dataset.}
As discussed in \cref{sec:intro}, estimation of $n$-gram language models relies on \emph{counting} the occurrences of various substrings in a dataset.
Given the dataset $\data = \{\str^{(\sampleidx)}\}_{\sampleidx=1}^\numsamples$, we define the counting function $\Count$ as
\begin{subequations}
\begin{align}
    \CountFun{\str} \defeq \sum_{m=1}^M \sum_{t=1}^{|\str| + 1}\sum_{s=t+1}^{|\str| + 1} \ind{\str = \str^{\left(\sampleidx\right)}_{t:s}}
     \\
    \CountFun{\str\eos} \defeq \sum_{m=1}^M \sum_{t=1}^{|\str|+1}\ind{\str = \str^{\left(\sampleidx\right)}_{t:}},
\end{align}\label{eq:countingfunction}
\end{subequations}
\!\!where $\str_{t:s} \defeq \word_{t} \cdots \word_{s-1}$ and $\str_{t:} \defeq \str_{t:|\str|+1}$.
As hinted at in~\cref{sec:intro}, $\Count$ counts the number of times the string $\str$ appears as a substring of a string in $\dataset$.\footnote{Note that, for complete generality, $\Count$ should also take the dataset $\dataset$ as argument. For conciseness, we leave this parameter implicit.}
This is to be distinguished from simply counting the number of occurrences of $\str$ in $\dataset$.

\paragraph{Empirical Distributions.}
Using the counting functions defined in~\cref{eq:countingfunction}, we define two empirical probability distributions.
The \defn{autoregressive empirical probability} distribution is defined as
\begin{equation}
    \pemp(\wordtwo\mid\str) \defeq \frac{\CountFun{\str\wordtwo}}{\CountFun{\str}},
    \label{eq:autoregressiveempirical}
\end{equation}
where $\str_{<t} \in \kleene{\alphabet}$ and $\wordtwo \in \alphabeteos$.~\cref{eq:autoregressiveempirical} is the autoregressive decomposition of~\cref{eq:empirical}.
Then, the \defn{autoregressive empirical $n$-gram probability} distribution is defined as follows
\begin{equation}
    \pngramemp(\wordtwo\mid\history) \defeq \frac{\CountFun{\history\wordtwo}}{\CountFun{\history}},
\end{equation}
where $\history \in \bosstringsn$ and $\wordtwo \in \alphabeteos$. 
Note that the autoregressive empirical $n$-gram probabilities are equivalent to the maximum-likelihood estimator of an $n$-gram model $\qngram$, as seen in~\cref{eq:qngram}.

\paragraph{Prefix Probabilities.} 
Prefix probabilities~\cite[\S2.4.2]{cotterell2023formal} are useful quantities in language modeling and are found both in the language model's autoregressive factorization as well as in maximum-likelihood estimation~\citep{jelinek-lafferty-1991-computation, nowak-cotterell-2023-fast}.
We give a formal definition below.
\begin{definition}\label{def:prefix}
    We define the \defn{prefix probability} function $\ppre$ of a language model $\p$ over $\kleene{\alphabet}$ as
    \begin{equation}\label{eq:prefix-prob}
        \ppre(\str) \!\defeq \! \sum_{\strtwo \in \kleene{\alphabet}} \p(\strtwo) \mathbbm{1}\{ \str \preceq \strtwo\} = \sum_{\strtwo \in \kleene{\alphabet}} \p(\str\strtwo),
    \end{equation}
    where $\str \preceq \strtwo$ indicates that $\str$ is a prefix of $\strtwo$.
\end{definition}

We now relate the MLE of a language model to the matching of \emph{next-symbol} conditional distributions.
It will later allow us to reason about the relationship between smoothing and regularized maximum-likelihood estimation.

\begin{restatable}{theorem}{kllocalglobal} \label{thm:ryans-theorem}
Let $\lmone$ and $\lmtwo$ be two language models over $\alphabet$ and $\ppre$ the prefix probability function of $\lmone$. 
Furthermore, we assume that $\entropy(p, q) < \infty$.
Then, the following equality holds
\begin{equation}
\begin{aligned}
\KL&(\lmone \mid\mid \lmtwo ) \label{eq:globallocalkl}\\
&= \sum_{\str \in \kleene{\alphabet}} \!\ppre(\str)\KL(\lmone(\cdot \mid \str) \mid\mid \lmtwo(\cdot \mid \str)).
\end{aligned}
\end{equation}
\end{restatable}
\begin{proof}
    \cref{app:prefix_and_kl}.
\end{proof}
\noindent We also consider the following corollary for $\lmone = \pemp$ and an $n$-gram language model $q$.
\begin{restatable}[]{corollary}{kldivml}
    Let $\pemp$ be an empirical distribution induced by a dataset $\dataset$.
    Let $q$ be an $n$-gram language model.
    Then, it holds that:
    \begin{align} \label{eq:ml-estimate}  
         \KL& \left(\pemp \mid\mid q \right) \\
         \propto & \sum_{\nstr \in \bosstringsn} \!\!\!\!\! \CountFun{\nstr} \KL\left( \pemp(\cdot \mid \nstr) \mid\mid q(\cdot \mid \nstr) \right) \nonumber.
    \end{align}
\end{restatable}
\begin{proof}
    \cref{app:smoothing-reg-proofs}.
\end{proof}
\noindent Crucially, the $n$-gram assumption allows us to reduce the infinite sum over $\kleene{\alphabet}$ in the definition of $\KL \left( \pemp \mid\mid \qngram \right)$ to one over $\bosstringsn$ due to the resulting conditional independence of $\word_t$ and $\word_{1} \cdots \word_{t - n}$ given $\historyt = \word_{t - n + 1} \cdots \word_{t - 1}$. 
Further, it allows us to restrict the summation to the $n$-grams that are present in the training corpus.\looseness=-1 

\subsection{Label Smoothing of $n$-gram LMs}
Let $\qtheta$ be a language model parametrized by parameters $\vtheta \in \Theta$. 
We further assume that $\qtheta$ is a differentiable function in $\vtheta$ and that $\Theta$ is compact.
Optimizing the objective in \cref{eq:ml-estimate} gives us the maximum-likelihood estimate of model parameters $\vtheta$.
To prevent overfitting and improve generalization abilities, this estimation can be regularized.
\vspace{-15pt}
\begin{principle}\label{def:regularization}
    The principle of \defn{regularization} states that we should add an inductive bias to the parameter estimation procedure, the goal of which is to help the model generalize to unseen data at the expense of its ability to better fit the training data.\looseness=-1 
\end{principle}

By \Cref{def:regularization}, we can intuitively see that smoothing techniques are a form of regularization. 
However, they form regularization that is defined procedurally in terms of the manipulation of count-based estimates. 
\defn{Label smoothing}, on the other hand, is defined as an additive augmentation of the \emph{training objective}---it represents the addition of the following regularizer to the training objective
\begin{equation}
    \regularizerls(\vtheta, \str) = \KL\left(\ru \mid\mid \qtheta(\cdot \mid \str)\right),
\end{equation}
where $\ru = \textstyle 1/|\eosalphabet|\cdot\mathbf{1}$ is the uniform distribution over $\eosalphabet$.
In words, label smoothing regularizes the maximum-likelihood objective toward a uniform distribution over the next symbol.
In the case of an $n$-gram model, we have the regularized objective
\begin{align} \label{eq:label-smoothing-reg}
        \sum_{\history \in \bosstringsn}\!\!\!\CountFun{\history} \Big[ \KL &\Big( \pemp(\cdot \mid \history) \mid\mid \qtheta(\cdot \mid \history) \Big)\nonumber \\
        & + \regPar \regularizerls(\vtheta, \history) \Big].
    \end{align}
The optimum of \cref{eq:label-smoothing-reg}, $\qsmoothTheta$, is then what we refer to as the label-smoothed version of the maximum-likelihood solution.\footnote{Throughout the paper, we will use the notation $\qsmooth$ for the smoothed version of the distribution $q$.} 

In \Cref{sec:intro}, we introduced add-$\lambda$ smoothing of \emph{$n$-gram} language models as a way to improve their generalization. However, regularization of the form \cref{eq:label-smoothing-reg} can be applied to \emph{any} language model $\qtheta$ whose parameters are learned through standard maximum-likelihood estimation---usually via gradient descent.
Attractively, we can show that, if $\qtheta$ is a $n$-gram language model, regularization from \cref{eq:label-smoothing-reg} is equivalent to add-$\lambda$ smoothing of $n$-gram counts in the sense that its optimum recovers the same model.
We believe this to be the first formal connection made between add-$\lambda$ smoothing and label smoothing of $n$-gram language models.\footnote{The restriction to $n$-gram language models is natural since the simple nature of $n$-gram language models permits the augmentation with hallucinated substring counts. 
Later in the paper, we show how this can be translated to neural language models by pre-processing the empirical data distribution.}
\vspace{-10pt}
\begin{restatable}[]{theorem}{lsreg} \label{thm:label-smoothing-add-lambda-ngrams}
    Estimating an $n$-gram model under regularized MLE with regularizer  $\regularizerls$ with strength parameter $\regPar$ is equivalent to estimating an $n$-gram model and applying add-$\lambda$ smoothing with $\lambda = \frac{\regPar}{|\alphabet| + 1}$.\looseness=-1
\end{restatable}
\begin{proof}
    \cref{app:smoothing-reg-proofs}
\end{proof}
\cref{thm:label-smoothing-add-lambda-ngrams} establishes an interpretable connection between a \emph{smoothing technique}---in this case, add-$\lambda$ smoothing, which can only be performed in the context of $n$-gram language models---and an \emph{additive regularizer}, which can be applied to general language models.
While additive regularizers are not common in the context of $n$-gram language models, where an augmentation of counts is usually more appropriate, this framing will facilitate the connection to more modern neural language models, as we showcase in \cref{sec:new-regularizers}. 

\section{Smoothing \texorpdfstring{$n$}{n}-Gram Counts}\label{sec:smoothing}
In the context of $n$-gram models, smoothing procedures generally modify the count-based MLE computation to address the fact that not all $n$-grams  occur in the training data.
We follow~\citet{chen_empirical} and review four well-known smoothing techniques of $n$-gram language models before connecting them to a generalized framework of regularization in \cref{sec:new-regularizers}.
\subsection{Good--Turing (\citeyear{good_turing})}
Good--Turing (GT) smoothing is one of the earliest methods devised to compute a smoothed $n$-gram model $\qngramsmooth$ from a $n$-gram model $\qngram$.
GT smoothing assigns cumulative probability mass to $n$-grams that appear $i$ times in the training data to be equal to the total probability mass of $n$-grams that appear $i+1$ times in the training data. To do so, adjusted $n$-gram counts $\CountGT(\nstr\word)$ are computed as\looseness=-1
\begin{equation}
    \CountGT(\nstr\word) = (\CountFun{\nstr\word}+1) \frac{\COfC_{\CountFun{\nstr\word}+1}}{\COfC_{\CountFun{\nstr\word}}},\label{eq:gt_counts}
\end{equation}
where $\COfC_{\CountFun{\nstr\word}}$ is the total number of $n$-grams that occur $\CountFun{\nstr\word}$ times in the training data, i.e., $\COfC_i \defeq \sum_{\nstr \in \bosstringsn} \mathbbm{1}\{\CountFun{\nstr\word} = i\}$. The probability of $\nstr$ is then defined as
\begin{equation}
    \pGT(\word\mid\nstr) = \frac{\CountGT(\nstr\word)}{\sum_{i=1}^{\infty} i \COfC_i}, \label{eq:gt_prob}
\end{equation}
where the denominator in~\cref{eq:gt_prob} is equivalent to the total number of tokens in $\dataset$. Note that any symbol whose successive count of counts is null is also assigned a null smoothed count. To avoid this issue,~\citet{simpleGT} propose to interpolate the missing counts through linear regression and use the regressed counts to compute the smoothed probabilities.

\subsection{Jelinek--Mercer (\citeyear{jelinek})}
Jelinek--Mercer (JM) smoothing relies on interpolation between higher-order and lower-order $n$-gram models to smooth $\qngram$.
The interpolation is applied recursively according to the following convex combination
\begin{equation}
\begin{aligned}
 \pJM(\word\mid\nstr) &= \lambda_n \qngram(\word\mid\nstr) \\
    &+ (1-\lambda_n) \pJMminus(\word\mid\nmstr).
\end{aligned}
\end{equation}
The recursion can be grounded either at the unigram level or with a uniform distribution over $\eosalphabet$. 

\subsection{Katz (\citeyear{katz})}
Katz smoothing relies on smoothed counts to compute its smoothed probabilities. 
These counts are computed as follows
\begin{equation}
\!\!\!\CountKZ(\nstr\word) \defeq 
\begin{cases}
    &\!\!\!\!\CountKZ(\nstr\word) \\
    & \quad \ifcondition \CountKZ(\nstr\word) > k \\ 
    &\!\!\!\!d_{\CountKZ(\nstr\word)} \CountKZ(\nstr\word) \\
    & \quad \ifcondition 0 < \CountKZ(\nstr\word) \leq k \\
    &\!\!\!\!\alpha(\nstr)\qngram(\nmstr\word) \\
    & \quad \otherwisecondition,
\end{cases}
\end{equation}
where $k$ is a hyperparameter whose value is usually assigned to a high-range, single-digit integer. 
For large counts, smoothed counts are equivalent to the empirical $n$-gram counts as the latter are assumed to be reliable.
Small non-zero counts, however, are discounted using count-specific discount factors $d_{\CountKZ}$, which are derived from the Good--Turing counts in~\cref{eq:gt_counts} and computed as
\begin{equation}
    d_{\CountKZ(\nstr\word)} = \frac{\frac{\CountGT(\nstr\word)}{\CountFun{\nstr\word}}- \frac{(k+1)\COfC_{k+1}}{\COfC_1}}{1-\frac{(k+1)\COfC_{k+1}}{\COfC_1}}.
\end{equation}
The total amount obtained by discounting is then redistributed to the $n$-grams with null counts, weighted by the probability of the lower-order $n$-gram and according to the normalization factor\looseness=-1
\begin{equation}
    \!\alpha(\nstr) = \frac{1-\sum_{\word: \CountFun{\nstr \word}>0} \pK(\word\mid\nstr)}{1-\sum_{\word: \CountFun{\nstr \word}>0} \qngram(\nmstr\word)}.
\end{equation}
Finally, smoothed probabilities are computed by normalizing the smoothed counts according to the following formula
\begin{equation}
    \pK(\word\mid\nstr) = \frac{\CountKZ(\nstr\word)}{\sum_{y\in\alphabeteos} \CountKZ(\nstr y)}.
\end{equation}

\subsection{Kneser--Essen--Ney \texorpdfstring{(\citeyear{kneser})}{}}

Kneser--Essen--Ney (KEN) smoothing is similar to Katz smoothing in that it also computes higher-order $n$-gram probabilities as a function of lower-order $n$-gram probabilities. 
However, in contrast to other smoothing methods, KEN smoothing does not construct $n$-gram probability distributions using simple counts, but rather using \emph{type counts}.
The type count of an $n$-gram is defined as the number of \emph{distinct} histories that the $n$-gram follows, rather than the absolute number of its occurrences in the data. 
Formally, the type count function $\typecount$ of order $n$ is defined as
\begin{equation}
    \typecount(\nstr, \word) \defeq \begin{cases}
        1 & \textbf{if } \#(\nstr \word) > 0 \\
        0 & \textbf{otherwise}.
    \end{cases}
\end{equation}
We further define $\typecount(\bullet, \word)$, $\typecount(\nstr, \bullet)$ and $\typecount(\bullet, \bullet)$ as the type count function with bulleted arguments summed out.
Type counts allow us to reduce the probability assigned to $n$-grams that occur many times in the data, but whose constituent $(n-1)$-grams have low probability.
A common illustrative example in the literature~\cite{chen_empirical} is the bigram \textword{San Francisco}. 
If the term \textword{San Francisco} appears frequently in a dataset, then the unigram probability assigned to \textword{Francisco} by smoothing methods that rely on lower-order $n$-gram distributions to compute higher-order $n$-gram distributions will be quite high. However, this is arguably undesirable in many situations, since the unigram \textword{Francisco} does not often appear after words other than \textword{San}.\looseness=-1

At the unigram level, the probability estimates of a KEN-smoothed distribution are computed using type counts for unigrams and bigrams as
\begin{equation}
   \!\! \pKNunigram(\word) = \frac{ \#_{\textsc{t}}^1(\bullet, \word)}{  \#_{\textsc{t}}^1(\bullet, \bullet)}.
\end{equation}
These probabilities are then used to ground the recursion that computes the smoothed probabilities $\pKN$ for higher-order $n$-grams according to the following formula
\begin{equation}
\begin{aligned}
    \pKN&(\word\mid\nstr) = \frac{\max\{ \CountFun{\nstr\word} - D, 0\}}{\sum_{y\in\alphabeteos} \CountFun{\nstr y}} \\
    & + \frac{D \cdot \typecount(\nstr\bullet) \cdot \pKNminus(\word\mid\nmstr)}{\sum_{y\in\alphabeteos}\CountFun{\nstr y}}.
\end{aligned}
\end{equation}
\section{A Generalized Framework}\label{sec:new-regularizers}
In \Cref{background:label-smoothing}, we evince a connection between add-$\lambda$ smoothing and regularization of the maximum-likelihood objective.
However, the derivation we formalize (cf. \cref{app:prefix_and_kl}) is tedious and long.
Moreover, it exploits several specific properties of add-$\lambda$ smoothing.
Performing such a derivation for each smoothing technique individually would be laborious and further, it would hinder building intuitions about the relationships between different methods.
Luckily, we can introduce a more general framework.
Specifically, in this section, we propose a framework that allows us to formulate equivalent regularizers for \emph{any} smoothing technique and apply them to the training of neural language models.\looseness=-1

\subsection{$n$-Gram Smoothing as Regularization}
Without further ado, we now introduce our framework for connecting the smoothing of $n$-gram language models to the regularization of the maximum-likelihood objective. 
This allows us to expand the notion of $n$-gram smoothing to neural language models. 
To this end, we first revisit MLE.

One way of framing MLE is using the KL divergence.
Specifically, given an empirical distribution $\pemp$, the principle of MLE dictates that we should choose a model $\qtheta$ such that $\KL\left(\pemp \mid\mid \qtheta\right) = 0$.
In comparison, $n$-gram smoothing techniques are often not defined so declaratively. 
Instead, they are presented as procedures that directly modify the empirical counts derived from a large dataset (e.g., \cref{eq:add-lambda-ml} in the simple case of add-$\lambda$ smoothing).
The crucial observation in this work is that we can treat $n$-gram smoothing
as a two-step process.
First, we view the smoother as a map $\pngramemp \mapsto \smoothpngramemp$ that outputs a smoothed empirical $n$-gram distribution.
Then, we choose the $\qtheta$ that minimizes $\KL\left(\smoothpngramemp\mid\mid \qtheta\right)$ where we have replaced $\pemp$ with $\smoothpngramemp$. 

In that context, the question we ask is this: Rather than minimizing $\KL\left(\smoothpngramemp\mid\mid \qtheta\right)$, can we always find a regularizer $\regularizer(\vtheta)$ such that $\KL\left(\smoothpngramemp\mid\mid \qtheta\right) = \KL\left(\pemp\mid\mid \qtheta\right) + \regularizer(\vtheta)$?
Such a result would be a natural generalization of the add-$\lambda$ case, discussed in \cref{thm:label-smoothing-add-lambda-ngrams} that would apply to any $n$-gram smoothing techniques, including all of those presented in \cref{sec:smoothing}.

\subsection{Smoothing as Regularization}
Now we turn to the primary question of this paper. 
How do we construct a regularizer that corresponds to an \emph{arbitrary} $n$-gram smoothing technique?
We begin by defining the following two probability distributions that together capture the difference between the empirical distribution and the smoothed empirical $n$-gram distribution:
\begin{subequations}
    \begin{align} 
        \ppos(\str) & \defeq \frac{1}{\lambdapos}  \max(0, \smoothpngramemp(\str) - \pemp(\str)) \label{eq:ppos} \\
        \pneg(\str) & \defeq \frac{1}{\lambdaneg}  \max(0, \pemp(\str) - \smoothpngramemp(\str)), \label{eq:pneg}
    \end{align}
\end{subequations}
where the normalization constants are defined as
\begin{subequations}
    \begin{align} 
        \lambdapos & \defeq \sum_{\str \in \kleene{\alphabet}} \max(0, \smoothpngramemp(\str) - \pemp(\str)) \label{eq:ppos-normalizer} \\
        \lambdaneg & \defeq \sum_{\str \in \kleene{\alphabet}} \max(0, \pemp(\str) -\smoothpngramemp(\str)) \label{eq:pneg-normalizer}.
    \end{align}
\end{subequations}
This results in the following simple decomposition:
\begin{equation} \label{eq:psmooth-decomp}
    \smoothpngramemp(\str) = \pemp(\str) + \lambdapos \ppos(\str) - \lambdaneg \pneg(\str).
\end{equation}
Why does the above formulation help?
Fundamental to our derivation in \Cref{background:label-smoothing} was the idea that we could think of add-$\lambda$ smoothing as adding a regularization term to the maximum-likelihood objective that penalizes diverging from a simple distribution---in the case of add-$\lambda$ smoothing, the uniform distribution over $\eosalphabet$.
Similarly, the decomposition of $\smoothpngramemp$ given in \Cref{eq:psmooth-decomp} facilitates the interpretation of training a language model on  $\smoothpngramemp$ as training $\pemp$ with the addition of regularization.
Concretely, we define the following regularizer
\begin{equation}\label{eq:generalized-reg}
\begin{aligned}
    \regularizer(\vtheta) \defeq \,&\lambdapos\KL( \ppos \mid\mid \qtheta) \\
     + &\lambdaneg  \KL( \pneg  \mid\mid \qtheta).
    \end{aligned}
\end{equation}
Now, the relation between estimating an $n$-gram model with a smoothing technique and using the regularizer formalized in \cref{eq:generalized-reg} is given by the following theorem.
\begin{restatable}[]{theorem}{genframework} \label{thm:genframework}
        Let $\pemp$ be the empirical distribution induced by the dataset $\dataset$ and $\smoothpngramemp$ a smoothed empirical $n$-gram distribution. 
    For $\gamma = 1$, the following holds
    \begin{equation}
    \begin{aligned}\label{eq:smoothing_decomposition}
    \!\!\KL( \smoothpngramemp &\mid\mid \qtheta) \\
    &= \KL( \pemp \mid\mid \qtheta) + \gamma  \regularizer(\vtheta) + C,
    \end{aligned}
    \end{equation}
where $C$ is constant with respect to $\qtheta$.
\end{restatable}
\begin{proof}
    \cref{app:smoothing-reg-proofs}
\end{proof}

\cref{thm:genframework} formalizes how training on the smoothed distribution $\smoothpngramemp$ computed by smoothing the $n$-gram counts affects the maximum-likelihood objective.
It brings us to an interesting observation about smoothing methods \emph{in general}---they can all be formalized as solutions to a regularized maximum-likelihood objective. 
Inspecting \cref{eq:smoothing_decomposition}, we see that, crucially, only the first term depends on the \emph{original} empirical distribution $\pemp$---indeed, it represents the original maximum-likelihood objective. 
The other two terms depend both on the empirical data distribution as well as its smoothed variant.
We can therefore interpret~\cref{eq:smoothing_decomposition} as a regularized loss where the last two terms correspond to the equivalent regularizer of the smoothing method used to construct $\smoothpngramemp$.
In practice, we might want to \emph{modulate} the strength of the regularization towards the smoothed distribution $\smoothpngramemp$. 
We can achieve such an effect through an additional hyperparameter $\regPar$, by which we multiply our regularizer in \cref{eq:smoothing_decomposition} to control its influence. 
The regularized loss can be decomposed further by \emph{splitting} the $\regPar$ hyperparameter into two terms \textcolor{ETHBlue}{$\regParPos$} and \textcolor{ETHBlue}{$\regParNeg$} and applying them separately to the positive and negative terms of the regularizer $\regularizer$
\begin{equation}
\begin{aligned}  \label{eq:approx_obj}
    \!\KL( \pemp \mid\mid \qtheta) + &\textcolor{ETHBlue}{\regParPos} \textcolor{ETHBlue}\lambdapos \KL(\ppos \!\mid\mid \qtheta) \\
    + &\textcolor{ETHBlue}{\regParNeg} \textcolor{ETHBlue}\lambdaneg \KL(\pneg\!\mid\mid \qtheta),
\end{aligned}
\end{equation}
where~\cref{eq:approx_obj} is equivalent to $\KL(\smoothpngramemp \mid\mid \qtheta)$ when both $\regParPos$ and $\regParNeg$ are equal to $1$.

Our generalized framework, therefore, presents a novel way of constructing regularizers to be used in the language modeling objective based on insights from classical methods for smoothing $n$-gram language models.
Importantly, it provides a direct mechanism by which smoothing-based regularization can be applied to any language model $\qtheta$. 
In the following section, we use this framework to explore the empirical effects of using regularizers constructed from smoothing methods (cf. \cref{sec:smoothing}) in the training of neural language models.\looseness=-1

\paragraph{Runtime Analysis.}
The distributions $\ppos$ and $\pneg$ require $\bigO{|\alphabet|^n}$ space to represent where $n$ is the $n$-gram order,
i.e., the space complexity is of the order of the number of $n$-gram contexts in the model. 
While the exponential increase in $n$ of $\bigO{|\alphabet|^n}$ is one of the main limitations for scalability of $n$-gram models, our method does not require increasing $n$ to large values, as it leverages (smoothed) $n$-gram models only in the construction of a regularizer for the training of a much larger neural model. 
Further, the scalability issues of $n$-gram models can be circumvented by using bespoke data structures~\cite{liu2024infini}.

\section{Experiments}
\subsection{Setup}

We validate our proposed regularization framework on two tasks: language modeling and machine translation.
We rely on the small-scale WikiText-2~\cite{Wikitext2} and IWSLT-14~\cite{IWSLT14} data sets, respectively, and compare the performance of standard MLE and label smoothing to the performance obtained by using regularizers based on the smoothing methods illustrated in~\cref{sec:smoothing}. 
For both tasks, we perform our experiments via the \texttt{fairseq} library~\cite{ott2019fairseq} on Transformer-based~\cite{transformers} language models.\looseness=-1

Our implementation of (Simple) Good--Turing smoothing in \texttt{fairseq} builds on an open-source implementation,\footnote{\url{github.com/maxbane/simplegoodturing}} while we leverage the efficient implementation of Kneser--Essen--Ney smoothing available through the \texttt{KenLM}~\cite{heafield-2011-kenlm, heafield-etal-2013-scalable} library.\footnote{\url{github.com/kpu/kenlm}}
The remaining smoothing methods were implemented natively in \texttt{fairseq}.
Note that, as all data sets are small in scale, we limit the maximum $n$-gram order to 2 (i.e., bigrams) for all smoothing methods.
We use dropout for all experiments fixing the dropout probability to $0.1$ and $0.3$ for language modeling and machine translation, respectively.
For all smoothing techniques, we set $\regParSpace \defeq \{0.005, 0.01, 0.05, 0.1, 0.5\}$ and grid search regularization hyperparameter pairs $\regParPos, \regParNeg \in \regParSpace \times \regParSpace$.
For smoothing methods that have additional hyperparameters, we extend the grid search described above to include them. We provide the complete list of method-specific hyperparameter values in~\cref{tab:hyperparams}. Additional dataset details are provided in~\cref{tab:dataset_details}. \looseness=-1

\begin{table}[]
\centering
\begin{tabular}{l|l}
    \toprule
    Smoothing Method & ppl $\downarrow$ \\
    \midrule
    None & $ 147.12 \dd{0.34}$           \\
    add-$\lambda$ {$\pars{(\regParPos=0.1, \regParNeg=0.05, \regParLS=0.01)}$}                         & $ 142.10^{\dagger} \dd{0.65}$ \\
    GT  {$\pars{(\regParPos=0.1, \regParNeg=0.05)}$} & $ \underline{141.93}^{\dagger} \dd{0.73}$          \\
    JM  {$\pars{(\regParPos=0.1, \regParNeg=0.5, \lambda_1=0.75)}$} & $ \textbf{137.41}^{\dagger} \dd{0.40}$         \\
    Katz  {$\pars{(\regParPos=0.1, \regParNeg=0.01, k=5)}$} & $ 142.69^{\dagger} \dd{0.54}$  \\
    KEN  {$\pars{(\regParPos=0.1, \regParNeg=0.1)}$} & $ 142.30^{\dagger} \dd{0.29}$ \\
  \bottomrule
\end{tabular}
\caption{Perplexity on WikiText-2 test set. Included are performances of models trained with no regularization (None), and with various smoothing methods. Reported perplexities are mean values for $5$ independently trained models, together with their standard errors. The best-performing method is in bold, while the second-best is underlined. $\dagger$ indicates statistical significance with respect to the unregularized baseline with $p<0.05$.\looseness=-1}\label{tab:WikiText_results}
\end{table}

\subsection{Language Modeling}

For language modeling, we evaluate the performance of our regularizers on the raw version of the WikiText-2 dataset \cite{Wikitext2} which we preprocess to remove all empty samples. 
We tokenize the data using BPE \cite{BPE} with \num{16000} merge operations through the \breaktexttt{subword-nmt} library.\footnote{\url{github.com/rsennrich/subword-nmt}}
For modeling, we use the decoder-only Transformer architecture denoted as \breaktexttt{transformer-lm} in \texttt{fairseq} while adopting standard hyperparameter settings as suggested by \texttt{fairseq}\footnote{\url{github.com/facebookresearch/fairseq/tree/main/examples/language_model}} to encourage reproducibility.
We train all models using early stopping and take as the best-performing models the ones with the lowest perplexity on the validation set.
For each method, the best-performing hyperparameter setting is then trained over 5 different seeds.
We summarize the results for the best-performing hyperparameter settings in~\cref{tab:WikiText_results}. 
We find that all regularized objectives outperform the unregularized baseline, with Jelinek--Mercer obtaining the lowest perplexity.
We test for mean separation using the Wilcoxon rank-sum test finding that all smoothing methods obtain statistically significant improvements over the unregularized baseline.
Perplexity scores for the best-performing runs are shown in~\cref{app:additional_experiments} together with $p$-value under a paired permutation test.

\begin{table}[]
\centering
\begin{tabular}{l|l}
    \toprule
       Smoothing Method           & BLEU $\uparrow$ \\
    \midrule
    None & $32.86 \dd{0.04}$       \\
    add-$\lambda$ {$\pars{(\regParPos=0.1, \regParNeg=0.01, \regParLS=0.01)}$} & $33.23^{\dagger} \dd{0.03}$  \\
    GT {$\pars{(\regParPos=0.05, \regParNeg=0.5)}$} & $33.37^{\dagger} \dd{0.01}$  \\
    JM {$\pars{(\regParPos=0.1, \regParNeg=0.5, \lambda_1=0.5)}$} & $\mathbf{33.67}^{\dagger} \dd{0.05}$  \\
    Katz {$\pars{(\regParPos=0.1, \regParNeg=0.1, k=5)}$} & $33.23^{\dagger} \dd{0.02}$  \\
    KEN {$\pars{(\regParPos=0.1, \regParNeg=0.1)}$} & $\underline{33.38}^{\dagger} \dd{0.03}$  \\
  \bottomrule
\end{tabular}
\caption{BLEU on test set of IWSLT-14 DE-EN. Different regularized methods are compared to no regularization (None). Reported values are means over $5$ independently trained models together with their standard errors. The best-performing method is in bold, while the second-best is underlined. $\dagger$ indicates statistical significance with respect to the unregularized baseline with $p<0.05$.}
\label{tab:iwslt_results}
\end{table}

\subsection{Machine Translation}
We evaluate the performance of our proposed regularizers on machine translation on the German-to-English task of the IWLST-14 dataset.
In the translation setting, we limit the application of smoothing only to distributions over the vocabulary of the target language.
We preprocess the data set by following the processing script provided by \texttt{fairseq}\footnote{\url{github.com/facebookresearch/fairseq/blob/main/examples/translation/prepare-iwslt14.sh}\looseness=-1} and tokenize the dataset with BPE using \num{10000} merge operations for both languages.
As our model, we use the small-sized \breaktexttt{transformer\_iwslt\_de\_en} encoder--decoder Transformer and its corresponding standard training hyperparameters.\footnote{\url{github.com/facebookresearch/fairseq/tree/main/examples/translation}}
We repeat the same grid search procedure over regularization hyperparameters as previously outlined, and use BLEU~\cite{papineni-etal-2002-bleu} on the validation set to determine the best-performing model checkpoints. 
To decode text from the model, we use beam search with a beam size of $5$.
We evaluate the generated translations with \texttt{sacreBLEU}~\cite{post-2018-call}.\footnote{\url{github.com/mjpost/sacrebleu}}\textsuperscript{,}\footnote{SacreBLEU signature: \path{nrefs:1|case:mixed|eff:no|tok:13a|smooth:exp|version:2.3.2}}
\cref{tab:iwslt_results} contains our results.
All smoothing methods improve over the baseline (no regularizer), and Jelinek--Mercer smoothing is the best-performing technique.
We repeat the mean separation tests outlined in the language modeling subsection and find that all smoothing methods obtain statistically significant improvements over the unregularized baseline.
In~\cref{app:additional_experiments} we additionally show the results of the best-performing models for each method and test their significance using paired bootstrap resampling~\cite{koehn-2004-statistical}.
Further, in~\cref{app:wmt} we present the results of a preliminary evaluation of our methods on the English-to-German task of the larger WMT14 machine translation dataset.\looseness=-1

\begin{table}[]
\centering
\begin{tabular}{l|l}
    \toprule
       Method           & Hyperparameters \\
    \midrule
    add-$\lambda$ & $\regParLS \in \{0.01, 0.05, 0.1\}$   \\
    GT & None   \\
    JM & $\lambda_1 \in \{0.25, 0.5, 0.75\} $  \\
    Katz & $k \in \{5, 7, 10\}$  \\
    KEN & None  \\
  \bottomrule
\end{tabular}
\caption{Method-specific hyperparameters on which a grid search was performed for both tasks. Note that in Jelinek--Mercer, $\lambda_2$ is obtained following its normalization constraint.}
\label{tab:hyperparams}
\end{table}

\section{Related Work}

\paragraph{Hybrid Neural and $n$-gram Models.}
The relationship between neural networks and $n$-gram models has been explored in previous work.
For instance, \citet{bengio2000neural} famously introduced a neural parameterization of an $n$-gram model, achieving state-of-the-art results at the time. 
More recently, \citet{sun-iyyer-2021-revisiting} scaled  \citeposs{bengio2000neural} model on modern hardware and demonstrated a small performance increase on language modeling over a Transformer model using a hybrid $n$-gram--Transformer model.
\citet{schwenk_2006} explored interpolating neural and $n$-gram language models.
\citet{neubig-dyer-2016-generalizing} expanded \citeposs{schwenk_2006} approach by exploring various ways to combine neural and $n$-gram language models.\looseness=-1

\paragraph{Regularization.}
On the topic of smoothing-based regularization, \citet{adaptiveLabelSmoothing} propose dynamically adjusting the strength of label smoothing regularization based on the entropy of the model distribution and using an earlier version of the model as a regularizer.
In a similar vein, \citet{baziotis2020} propose using a monolingual language model as a regularizer for a translation model. 
The idea is that monolingual data is far more abundant than bilingual data, so a language model of the target language is used to guide the target distribution of the translation model.
\citet{peters-martins-2021-smoothing} generalize label smoothing to the broader family of Fenchel--Young losses, making it applicable to entmax-based models, while~\citet{meister2020generalized} generalize label smoothing to a set of entropy-based regularizers.

\section{Conclusion}
In this work, we re-imagine the application of classical $n$-gram smoothing techniques in the context of modern neural NLP models. 
For several of these historic methods, we derive equivalent, differentiable regularizers that can be added to neural models' training objectives. 
We present these results within a generalized framework that allows for insights about the smoothing methods themselves and their relationships to each other. 
We apply these smoothing methods in the training of neural language models and machine translation models.
We find that our smoothing-based regularizers outperform label smoothing and standard MLE in language modeling, while some methods also achieve competitive results with label smoothing for machine translation.\looseness=-1

\section*{Limitations}
We present results only for English (for language modeling) and between German and English (for machine translation).
Most experiments are limited to small datasets for both language modeling and machine translation.
Future work could verify how scaling the amount of data impacts present results and whether the observed performance improvements are also achievable in a wider set of languages.
While in some experimental settings, we observed performance improvements for some smoothing methods, the additional computational complexity required by their use may not be a worthwhile trade-off for their performance benefits.\looseness=-1

\section*{Ethics Statement}
This paper is theoretical in nature, as it aims to shed light on the relationship between $n$-gram smoothing methods and language model regularization.
For this reason, the authors foresee no ethical concerns with the research presented in this paper.

\section*{Acknowledgements}
We thank Li Du for helpful feedback on a draft of this paper. Anej Svete and Afra Amini are supported 
by the ETH AI Center Doctoral Fellowship.

\bibliography{custom}

\appendix

\onecolumn
\allowdisplaybreaks

\section{Proofs}\label{app:proofs}
This section contains the proofs of all the theorems in the main text.
We begin by discussing prefix probabilities in \cref{sec:pprobs} as a tool for analyzing the relationship between the full Kullback--Leibler divergence $\KL\left(p \mid \mid q\right)$ and the divergences between the conditional probabilities $\KL\left(p\left(\cdot \mid \str\right) \mid \mid q\left(\cdot \mid \str\right) \right)$ for $\str \in \kleene{\alphabet}$.
We then move on to the proofs of the results characterizing the aforementioned relationship in \cref{app:smoothing-reg-proofs}.

\subsection{Prefix Probabilities} \label{sec:pprobs}
\subsubsection{Introductory Notes on Prefix Probabilities}
In this section, we provide supplementary commentary on prefix probabilities.
First, proving the second equality in \cref{eq:prefix-prob} is a useful exercise; indeed, the last author often assigns the task to his students \cite{cotterell2023hw}.
Second, it is important to keep in mind that while $\ppre(\str)$ is the probability of a certain event---namely, the event that a string \emph{starts} with the prefix $\str$, $\ppre$ itself is not a valid probability distribution, i.e., $\sum_{\str \in \kleene{\alphabet}} \ppre(\str) \neq 1$.
Indeed, $\ppre$ may not even be normalizable, i.e., we may have that $\sum_{\str \in \kleene{\alphabet}} \ppre(\str) \rightarrow \infty$. 
This property should make intuitive sense: 
By the definition of $\ppre$, we count the probability of certain events under $\p$ in our computation of prefix probabilities under $\ppre$ \emph{multiple times}. 
For example, in the case that $\alphabet= \{\texttt{a}\}$,  $\p(\texttt{a})$ counts towards both $\ppre(\texttt{a})$ and $\ppre(\texttt{aa})$.
Since $\p$ is a valid probability distribution, i.e., its probabilities sum to 1, then our prefix probabilities will often sum to $>1$.

In the special case of empirical distributions, the prefix probability of a substring is proportional to the number of times the substring appears in the dataset $\dataset$.
Concretely, we have that
\begin{equation}\label{eq:empirical-prefix}
    \prefixemp(\str) \propto \text{number of strings in $\dataset$ starting with $\str$},
\end{equation}
which means that
\begin{subequations}
\begin{align} 
    \sum_{\strtwo \in \kleene{\alphabet}} \prefixemp(\strtwo\str) &\propto \sum_{\strtwo \in \kleene{\alphabet}}  \text{number of strings in $\dataset$ starting with $\strtwo\str$} \label{eq:counts-ppre-relationship-1}\\
    &\propto \CountFun{\str}\label{eq:counts-ppre-relationship-2}.
\end{align}
\end{subequations}
That is, the number of occurrences of $\str$ in $\dataset$ is proportional to $\sum_{\strtwo \in \kleene{\alphabet}} \prefixemp(\strtwo\str)$.
Why is \Cref{eq:counts-ppre-relationship-2} true? 
Because every time we observe a $\str$ in the training dataset it must have some prefix that starts with the beginning of a string.

\subsubsection{Prefix Probabilities and Local Kullback--Leibler Divergences}\label{app:prefix_and_kl}
We now move on to proving a crucial component of analyzing the relationship between dataset smoothing and regularized training---the relationship between the global Kullback--Leibler divergence and the local Kullback--Leibler divergences of the next-symbol conditional probabilities. 
Intuitively, we show that the global Kullback--Leibler divergence $\KL(p \mid\mid q )$ can be written as a prefix-probability-weighted sum of local next-symbol probability distributions $\KL\left(p\left(\cdot\mid\str\right) \mid\mid q\left(\cdot\mid\str\right) \right)$ for $\str \in \kleene{\alphabet}$.
This result, which we believe to be novel, is formally captured by \cref{thm:ryans-theorem}.

\kllocalglobal*
\begin{proof}
Let $\alphabet$ be an alphabet and let $\lmone$ and $\lmtwo$ be distributions over $\kleene{\alphabet}$. We make use of the following definition. Let $\str \in \kleene{\alphabet}$ be a string and $T \in \mathbb{N}_{\geq 0}$. The bounded prefix probability of $\str$ is defined as
\begin{equation}
\ppreT(\str) \defeq \sum_{\strtwo \in \kleene{\alphabet}} \mathbbm{1}\{\str \preceq \strtwo\} \lmone(\strtwo \mid T), \label{eq:boundedprefix}
\end{equation}
where $\str \preceq \strtwo$ indicates that $\str$ is a prefix of $\strtwo$, $\lmone(\strtwo \mid T)$ is the conditional of language model $\lmone$ to strings of length $T$, and $\lmone(T) = \sum_{\str \in \alphabet^T} \lmone(\str)$. Note that the following equality relates $\ppre$ and $\ppreT$
\begin{equation}
\ppre(\str) = \sum_{T=0}^\infty \lmone(T) \ppreT(\str). \label{eq:boundedprefixtoprefix}
\end{equation}
To prove~\cref{eq:globallocalkl}, we split $\KL$ into a cross-entropy and an entropy term as follows
\begin{subequations}
\begin{align}
&\KL(\lmone \mid\mid \lmtwo) \defeq \sum_{\str \in \kleene{\alphabet}} \lmone(\str)\log\left(\frac{\lmone(\str)}{\lmtwo(\str)}\right)\\
&= - \sum_{\str \in \kleene{\alphabet}} \lmone(\str)\log(\lmtwo(\str)) + \sum_{\str \in \kleene{\alphabet}} \lmone(\str)\log(\lmone(\str)) \\
&= \entropy(\lmone, \lmtwo) - \entropy(\lmone).
\end{align}
\end{subequations}
We show the equivalence of cross-entropy, as entropy is the special case when $\entropy(\lmone, \lmone)$.\footnote{In what follows, we will frequently interchange infinite sums. Since the terms involved are either all $\geq 0$ or all $\leq 0$, Tonelli's theorem guarantees that such interchanges are valid~\citep[Theorem 2.37.a applied to discrete measures]{folland1999}.} Starting with the cross-entropy we have
\begin{subequations}
\begin{align}
&\entropy(\lmone, \lmtwo) \defeq -\sum_{\str \in \kleene{\alphabet}} \lmone(\str)\log \lmtwo(\str) \label{eq:globallocalfirst}\\
&= -\sum_{\str \in \kleene{\alphabet}} \lmone(\str)\left[\log\left(\lmtwo(\eos \mid \str) \prod_{t=1}^T \lmtwo(\word_t \mid \str_{<t})\right) \right] \\
&=
 -\sum_{\str \in \kleene{\alphabet}} 
 \sum_{T=0}^\infty \lmone(T) \lmone(\str\mid T)
 \left[\log\left(\lmtwo(\eos \mid \str) \prod_{t=1}^T \lmtwo(\word_t \mid \str_{<t})\right) \right]
\\
&= -\sum_{T=0}^\infty  \lmone(T)  \sum_{\str \in \kleene{\alphabet}} \lmone(\str \mid T) \left[\log\left(\lmtwo(\eos \mid \str) \prod_{t=1}^T \lmtwo(\word_t \mid \str_{<t})\right) \right] \\
&= -\sum_{T=0}^\infty \lmone(T) \Bigg[ \sum_{\str \in \kleene{\alphabet}} \lmone(\str \mid T) \log \lmtwo(\eos \mid \str) \\
&\qquad \qquad \qquad + \sum_{t=1}^T \sum_{\str \in \kleene{\alphabet}} \lmone( \str \mid T) \log \lmtwo(\word_t \mid \str_{<t})\Bigg] \annot{(distribute $\log$ and $\lmone(\str \mid T)$)}\nonumber\\
&= -\sum_{T=0}^\infty \lmone(T)\left[\sum_{\str \in \kleene{\alphabet}} \lmone(\str \mid T) \log \lmtwo(\eos \mid \str) + \sum_{t=1}^T \sum_{\strtwo \in \kleene{\alphabet}}\sum_{\str_{\leq t} \in \alphabet^{t}} \lmone( \str_{\leq t}\strtwo \mid T) \log \lmtwo(x_t \mid \str_{<t})\right] \\
&= -\sum_{T=0}^\infty \lmone(T)\left[\sum_{\str \in \kleene{\alphabet}} \lmone(\str \mid T) \log \lmtwo(\eos \mid \str) + \sum_{t=1}^T \sum_{\str_{\leq t} \in \alphabet^{t}}  \log \lmtwo(\word_t \mid \str_{<t})\sum_{\strtwo \in \kleene{\alphabet}}\lmone(\str_{\leq t}\strtwo \mid T)\right]  \\
&= -\sum_{T=0}^\infty \lmone(T)\Bigg[\sum_{\str \in \kleene{\alphabet}} \lmone(\str \mid T) \log \lmtwo(\eos \mid \str) \\
&\qquad \qquad \qquad+ \sum_{t=1}^T \sum_{\str_{\leq t} \in \alphabet^{t}}  \log \lmtwo(\word_t \mid \str_{<t}) \ppreT(\str_{\leq t}) \Bigg] \annot{(definition of $\ppreT$)}\nonumber\\
&= -\sum_{T=0}^\infty \lmone(T)\left[\sum_{\str \in \kleene{\alphabet}} \lmone(\str \mid T) \log \lmtwo(\eos \mid \str) + \sum_{t=1}^T \sum_{\word \in \alphabet} \sum_{\str_{< t} \in \alphabet^{t-1}}  \log \lmtwo(\word \mid \str_{<t}) \ppreT(\str_{< t}\word) \right]\\
&= -\sum_{T=0}^\infty \lmone(T)\left[\sum_{\str \in \kleene{\alphabet}} \lmone(\str \mid T) \log \lmtwo(\eos \mid \str) + \sum_{\word \in \alphabet} \sum_{t=1}^T \sum_{\str_{< t} \in \alphabet^{t-1}}  \log \lmtwo(\word \mid \str_{<t}) \ppreT(\str_{< t} \word) \right]\\
&= -\sum_{T=0}^\infty \lmone(T)\left[\sum_{\str \in \kleene{\alphabet}} \lmone(\str \mid T) \log \lmtwo(\eos \mid \str) + \sum_{\word \in \alphabet}  \sum_{\str \in \alphabet^{< T}}  \log \lmtwo(\word \mid \str) \ppreT(\str \word) \right]\\
&= -\sum_{T=0}^\infty \lmone(T)\left[\sum_{\str \in \kleene{\alphabet}} \lmone(\str \mid T) \log \lmtwo(\eos \mid \str) + \sum_{\word \in \alphabet}  \sum_{\str \in \kleene{\alphabet}}  \log \lmtwo(\word \mid \str) \ppreT(\str \word) \right]\\
&= -\sum_{\str \in \kleene{\alphabet}} \sum_{T=0}^\infty \lmone(T) \lmone(\str \mid T) \log \lmtwo(\eos \mid \str) +  \sum_{\word \in \alphabet}\sum_{\str \in \kleene{\alphabet}}\sum_{T=0}^\infty \lmone(T) \ppreT(\str \word)  \log \lmtwo(\word \mid \str) \\
&= -\sum_{\str \in \kleene{\alphabet}} \lmone(\str) \log \lmtwo(\eos \mid \str) -  \sum_{\word \in \alphabet}\sum_{\str \in \kleene{\alphabet}} \ppre(\str \word) \log \lmtwo(\word \mid \str) \\
&= -\sum_{\str \in \kleene{\alphabet}} \ppre(\str) \lmone(\eos \mid \str) \log \lmtwo(\eos \mid \str) -   \sum_{\str \in \kleene{\alphabet}} \ppre(\str)  \sum_{\word \in \alphabet} \lmone(\word \mid \str) \log \lmtwo(\word \mid \str) \\
&= -\sum_{\str \in \kleene{\alphabet}} \ppre(\str) \sum_{\word \in \bar{\alphabet}} \lmone(\word \mid \str) \log \lmtwo(\word \mid \str) \\
&= \sum_{\str \in \kleene{\alphabet}} \ppre(\str)  \entropy\left(\lmone(\cdot \mid \str), \lmtwo(\cdot \mid \str)\right). \label{eq:entropy-final}
\end{align}
\end{subequations}
Now, we substitute \Cref{eq:entropy-final} into the following equation
\begin{subequations}
\begin{align}
&\KL(\lmone \mid\mid \lmtwo) = \sum_{\str \in \kleene{\alphabet}} \lmone(\str)\log\left(\frac{\lmone(\str)}{\lmtwo(\str)}\right)\\
&= \entropy(\lmone, \lmtwo) - \entropy(\lmone) \\
&= \sum_{\str \in \alphabet^{*}} \ppre(\str)  \entropy\left(\lmone(\cdot \mid \str), \lmtwo(\cdot \mid \str)\right) - \sum_{\str \in \alphabet^{*}} \ppre(\str)  \entropy\left(\lmone(\cdot \mid \str) \right) \\
&= \sum_{\str \in \alphabet^{*}} \ppre(\str)  \left(-\sum_{x \in \alphabeteos} \lmone(x \mid \str) \log \lmtwo(x \mid \str) \right ) - \sum_{\str \in \alphabet^{*}} \ppre(\str)  \left(-\sum_{x \in \alphabeteos} \lmone(x \mid \str) \log \lmone(x \mid \str) \right ) \\
&= \sum_{\str \in \alphabet^{*}} \ppre(\str)  \left(\sum_{x \in \alphabeteos} \lmone(x \mid \str) \log \frac{\lmone(x \mid \str)}{\lmtwo(x \mid \str)} \right) \\
&= \sum_{\str \in \alphabet^{*}} \ppre(\str)  \KL(\lmone(\cdot \mid \str) \mid\mid \lmtwo(\cdot \mid \str)).
\end{align}
\end{subequations}
Note that because we have assumed that $\entropy(\lmone, \lmtwo) < \infty$
and $\entropy(\lmone) \leq \entropy(\lmone, \lmtwo)$, we have that both additive terms in the KL divergence are finite.
This is sufficient to avoid $\infty - \infty$ and the unpleasantries that follow.\looseness=-1
\end{proof}

\subsection{Smoothing and Regularization} \label{app:smoothing-reg-proofs}

\kldivml*
\begin{proof}
Then, we have
 \begin{subequations}
\begin{align}
    \KL\left(\pemp \mid \mid \qngram \right) &= \sum_{\str \in \kleene{\alphabet}} \prefixemp(\str) \KL\left( \pemp(\cdot \mid \str) \mid\mid q(\cdot \mid \str) \right) &\annot{(\Cref{thm:ryans-theorem})} \\ 
    &\propto \sum_{\str \in \kleene{\alphabet}} \CountFun{\str} \KL\left( \pemp(\cdot \mid \str) \mid\mid q(\cdot \mid \str) \right) &\annot{(\Cref{eq:empirical-prefix})} \\ 
    &= \sum_{\str \in \kleene{\alphabet}} \CountFun{\str} \KL\left( \pemp(\cdot \mid \nstr) \mid\mid q(\cdot \mid \nstr) \right) &\annot{(\Cref{ass:ngram})}  \\ 
    &= \sum_{\nstr \in \bosstringsn} \CountFun{\nstr}  \KL\left( \pemp(\cdot \mid \nstr) \mid\mid q(\cdot \mid \nstr) \right),
\end{align}
\end{subequations}
in which the final manipulation follows by \Cref{ass:ngram} and rearranging the terms in the sum.
This finishes the proof.
\end{proof}

\lsreg*
\begin{proof}
    Let $q$ be an $n$-gram model.
    As introduced in~\cref{eq:add-lambda-ml}, add-$\lambda$ smoothing defines the following smoothed $n$-gram probability distribution for $\word \in \eosalphabet$ given a history $\history$:\footnote{Note that $|\eosalphabet| = |\alphabet| + 1$ due to the inclusion of the $\eos$ symbol. We use the more explicit notation $|\alphabet| + 1$ for clarity of exposition.}
    \begin{equation}
        \qngramsmooth(\word \mid \history) \defeq \frac{\CountFun{\history\word} + \lambda}{\CountFun{\history} + \lambda(|\alphabet| + 1)}. \label{eq:label-smoothing}
    \end{equation}
    We want to show that the $\lambda$-count-augmented maximum-likelihood solution $\qngramsmooth$ is also the optimum of the label smoothing objective.
   We first decompose the KL divergence, which is the objective we optimize under the principle of maximum likelihood
    \begin{subequations}
        \begin{align}
            &\KL\left(\pngramemp\mid\mid q\right) + \regularizerls(\vtheta) =\\
                & = \sum_{\history \in \bosstringsn} \CountFun{\history} \KL\left(\pngramemp\left(\cdot \mid \history\right) \mid\mid q\left(\cdot \mid \history\right)\right)  + \regPar \sum_{\history \in \bosstringsn} \KL\left(u\left(\cdot \mid \history\right) \mid\mid q\left(\cdot \mid \history\right)\right)                     \\
                & = \sum_{\history \in \bosstringsn} \!\!\!\left[\CountFun{\history} \KL\left(\pngramemp\left(\cdot \mid \history\right) \mid\mid q\left(\cdot \mid \history\right)\right) + \regPar \KL\left(u\left(\cdot \mid \history\right) \mid\mid q\left(\cdot \mid \history\right)\right)\right]                                                    \\
                & = \sum_{\history \in \bosstringsn} \Bigg[ \CountFun{\history} \Big[ \sum_{\word \in \eosalphabet}\pngramemp\left(\word \mid \history\right) \log q\left(\word \mid \history\right)\Big] \\
                & \qquad \qquad \qquad + \regPar \Big[\sum_{\word \in \eosalphabet}u\left(\word \mid \history\right) \log q\left(\word \mid \history\right)\Big]\Bigg] + \text{const.}  \nonumber \\
                & = \sum_{\history \in \bosstringsn} \sum_{\word \in \eosalphabet}            \left[\CountFun{\history}  \pngramemp\left(\word \mid \history\right) \log q\left(\word \mid \history\right) + \regPar u\left(\word \mid \history\right) \log q\left(\word \mid \history\right)\right]  + \text{const.}                                                  \\
                & = \sum_{\history \in \bosstringsn} \sum_{\word \in \eosalphabet} \left[\CountFun{\history}  \pngramemp\left(\word \mid \history\right) + \regPar u\left(\word \mid \history\right) \right] \log q\left(\word \mid \history\right)  + \text{const.}                                                                                                       \\
                & = \sum_{\history \in \bosstringsn} \sum_{\word \in \eosalphabet} \left[\CountFun{\history}  \frac{\CountFun{\history \word}}{\CountFun{\history}} + \frac{\regPar}{|\alphabet| + 1} \right] \log q\left(\word \mid \history\right) + \text{const.}                                                                                 \\
                & = \sum_{\history \in \bosstringsn} \sum_{\word \in \eosalphabet} {\color{ETHBlue} \left[\CountFun{\history \word} + \frac{\regPar}{|\alphabet| + 1} \right] }\log q\left(\word \mid \history\right) + \text{const.}, \label{eq:theorem   }
        \end{align}
    \end{subequations}
    where the constant terms are independent of $q$. 
    Next, note that we can optimize each $q\left(\word \mid \history\right)$ independently, i.e., we can find
    the distribution $q\left(\word \mid \history\right)$ that minimizes the following expression
    \begin{equation}\label{eq:optimum}
       {\color{ETHBlue} \left[\CountFun{\history \word} + \frac{\regPar}{|\alphabet| + 1} \right] }\log q\left(\word \mid \history\right),
    \end{equation}
    under the constraint that $\sum_{\word\in \alphabeteos}  q(\word \mid \history) = 1$ and $q(\word \mid \history) \geq 0$, $\forall \word \in \alphabeteos$ independently. 
    It is a standard result that the minimizing $q\left(\cdot \mid \history\right)$ for any $\history \in \bosstringsn$ is given by
    \begin{equation}
        \qngramsmooth\left(\cdot \mid \history\right)  =\frac{\CountFun{\history \word} + \frac{\regPar}{|\alphabet| + 1}}{\CountFun{\history} + \regPar} \propto  {\color{ETHBlue} \left[\CountFun{\history \word} + \frac{\regPar}{|\alphabet| + 1} \right] },
    \end{equation}
    in which we recognize the $\lambda = \frac{\regPar}{|\alphabet| + 1}$ add-$\lambda$ smoothed maximum-likelihood solution $\qngramsmooth$ from~\cref{eq:label-smoothing}.
\end{proof}

\genframework*
\begin{proof}
    The definitions of $\ppos, \pneg, \lambdapos, \lambdaneg$ from \cref{eq:ppos,eq:pneg,eq:psmooth-decomp,eq:ppos-normalizer}, respectively, results in the following simple decomposition:
    \begin{equation}
        \smoothpngramemp(\str) = \pemp(\str) + \lambdapos \ppos(\str) - \lambdaneg \pneg(\str).
    \end{equation}
Then, we proceed with some basic manipulations
    \begin{subequations}
    \begin{align}
         \KL(\smoothpngramemp& \mid\mid \qtheta)  \defeq \entropy(\smoothpngramemp, \qtheta) - \underbrace{\entropy(\smoothpngramemp)}_{\defeq C}\\
         & = \entropy(\smoothpngramemp, \qtheta) + C  \annot{(independence of $\entropy(\smoothpngramemp)$ with respect to $\qtheta$)}\\
         & = \entropy( \pemp + \lambdapos \ppos (\str) + \lambdaneg \pneg(\str), \qtheta)  + C \annot{(definitions of $\ppos, \pneg, \lambdapos, \lambdaneg$)}\\
         & = \entropy( \pemp, \qtheta) + \lambdapos \entropy( \ppos, \qtheta) + \lambdaneg \entropy( \pneg,  \qtheta) + C \annot{(linearity of cross-entropy)}\\
         & = \KL( \pemp \mid\mid \qtheta) + \underbrace{\lambdapos \KL( \ppos  \mid\mid \qtheta) + \lambdaneg \KL(\pneg \mid\mid \qtheta)}_{\defeq \regularizer(\qtheta)} + C\annot{(definition of $\regularizer$)}\\
         &= \KL( \pemp \mid\mid \qtheta) + \underbrace{\regPar}_{=1}\regularizer(\qtheta) + C. 
    \end{align}
    \end{subequations}
    This proves the result.
\end{proof}

\section{Experimental Details}

Experiments on WikiText-2 and IWSLT-14 were run on a shared cluster on NVIDIA Quadro RTX 6000 GPUs.
The Transformer models used for language modeling and machine translation have \num{58145792} and \num{39469056} parameters, respectively.

\begin{table}[h]
\centering
\begin{tabular}{l|l|l|l|l|l}
    \toprule
       Dataset & Split & Language & Vocabulary size & Samples & Number of tokens  \\
    \midrule
       WikiText-2 & Train & English & \num{16932} & \num{23767} & \num{2389674}  \\
       WikiText-2 & Validation& English& \num{16932} & \num{2461} & \num{255327} \\
       WikiText-2 & Test & English& \num{16932} & \num{2891} & \num{292710} \\
    \midrule
       IWSLT-14 & Train& English& \num{6628} & \num{160239} & \num{3788875} \\
       IWSLT-14 & Validation& English& \num{6628} & \num{7283} & \num{171339} \\
       IWSLT-14 & Test& English& \num{6628} & \num{6750} & \num{150178}  \\
    \midrule
       IWSLT-14 & Train& German & \num{8844} & \num{160239} & \num{3875352} \\
       IWSLT-14 & Validation & German& \num{8844} & \num{7283} & \num{175309}  \\
       IWSLT-14 & Test& German& \num{8844} & \num{6750} & \num{155088} \\
  \bottomrule
\end{tabular}
\caption{Dataset details}
\label{tab:dataset_details}
\end{table}

\section{Additional Experimental Results}\label{app:additional_experiments}
\cref{tab:best_WikiText_results} and~\cref{tab:best_iwslt_results} show the performance of the best-performing seed for each model. JM smoothing remains the best-performing technique on both datasets.
For language modeling, we test for statistical significance using a paired permutation test over sentence-level log-likelihoods, using the mean of the observation differences as the test statistic.
For machine translation, we test for statistical significance using paired bootstrap resampling as implemented in \texttt{sacreBLEU}. Different tests are performed using either the unregularized results or the add-$\lambda$ results as the baseline. In~\cref{tab:best_WikiText_results} and ~\cref{tab:best_iwslt_results}, the first symbol in $\dagger/\dagger$ refers to statistical significance with respect to the unregularized model, while the second refers to statistical significance over the add-$\lambda$ results. $\dagger$ indicates $p<0.05$; $\ddagger$ indicates $p<0.01$. 
In language modeling, all regularization methods perform significantly better than no regularization. However, only GT and JM smoothing perform better than add-$\lambda$ smoothing. For machine translation, we see that all regularized methods except for Katz smoothing perform significantly better than the unregularized baseline model, while only JM smoothing performs significantly better than add-$\lambda$ smoothing.

\begin{table}[h]
\parbox[t]{.45\linewidth}{
\centering
\begin{tabular}{l|l}
    \toprule
    Smoothing Method & ppl $\downarrow$ \\
    \midrule
    None & $144.67$           \\
    add-$\lambda$ {$\pars{(\regParPos=0.1, \regParNeg=0.05, \regParLS=0.01)}$}                         & $138.72^{\ddagger}$ \\
    \midrule
    GT  {$\pars{(\regParPos=0.1, \regParNeg=0.05)}$} & $137.00^{\ddagger/\ddagger}$          \\
    JM  {$\pars{(\regParPos=0.1, \regParNeg=0.5, \lambda_1=0.75)}$} & $134.63^{\ddagger/\ddagger}$         \\
    Katz  {$\pars{(\regParPos=0.1, \regParNeg=0.01, k=5)}$} & $139.44^{\ddagger}$  \\
    KEN  {$\pars{(\regParPos=0.1, \regParNeg=0.1)}$} & $140.71^{\ddagger}$ \\
  \bottomrule
\end{tabular}
\caption{Perplexity on WikiText-2 test set. Included are performances of models trained with no regularization (None), and with various smoothing methods. We report the best perplexity over $5$ independently trained models.\looseness=-1}\label{tab:best_WikiText_results}}
\hfill
\parbox[t]{.45\linewidth}{
\centering
\begin{tabular}{l|l}
    \toprule
       Smoothing Method           & BLEU $\uparrow$ \\
    \midrule
    None & $33.11$       \\
    add-$\lambda$ {$\pars{(\regParPos=0.1, \regParNeg=0.01, \regParLS=0.01)}$} & $33.41^{\dagger} $  \\
    \midrule
    GT {$\pars{(\regParPos=0.05, \regParNeg=0.5)}$} & $33.44^{\dagger}$  \\
    JM {$\pars{(\regParPos=0.1, \regParNeg=0.5, \lambda_1=0.5)}$} & $33.97^{\ddagger/\ddagger}$  \\
    Katz {$\pars{(\regParPos=0.1, \regParNeg=0.1, k=5)}$} & $33.31$  \\
    KEN {$\pars{(\regParPos=0.1, \regParNeg=0.1)}$} & $33.58^{\ddagger}$  \\
  \bottomrule
\end{tabular}
\caption{BLEU on test set of IWSLT-14 DE-EN. Different regularized methods are compared to no regularization (None). We report the best BLEU score over $5$ independently trained models.}
\label{tab:best_iwslt_results}}
\end{table}

\subsection{WMT-14 results}\label{app:wmt}

To evaluate the performance of our regularizers on larger datasets, we perform a preliminary evaluation of our methods on the WMT-14 machine translation dataset.
We download and preprocess the data following a script provided by \texttt{fairseq}.\footnote{\url{https://github.com/facebookresearch/fairseq/blob/main/examples/translation/prepare-wmt14en2de.sh}}
Experiments on WMT were run on two NVIDIA Tesla V100 GPUs.
Given our computational constraints, we limit our experiments to our best-performing smoothing method (Jelinek--Mercer), add-$\lambda$ smoothing, and the unregularized baseline. 
We also use the same hyperparameter set as in the IWSLT-14 experiments for all methods.
\cref{tab:wmt_results} shows the results of our experiments. JM smoothing obtains the best performance out of the three methods.
Using paired bootstrap resampling, we assessed that the improvements of JM smoothing are statistically significant over the unregularized baseline, while they are not significant over the results obtained using add-$\lambda$ smoothing ($p=0.07$).
A regularization hyperparameter search on the new dataset might yield larger performance improvements for all smoothing-based methods.

\begin{table}[h]
\centering
\parbox{.45\linewidth}{
\begin{tabular}{l|l}
    \toprule
       Smoothing Method           & BLEU $\uparrow$ \\
    \midrule
    None & $26.75$       \\
    add-$\lambda$ {$\pars{(\regParPos=0.1, \regParNeg=0.01, \regParLS=0.01)}$} & $27.07$  \\
    JM {$\pars{(\regParPos=0.1, \regParNeg=0.5, \lambda_1=0.5)}$} & $27.35^{\ddagger}$  \\
  \bottomrule
\end{tabular}
\caption{BLEU scores on the (EN-DE) \texttt{WMT14/full} test set. JM and add-$\lambda$ smoothing are compared to no regularization.}\label{tab:wmt_results}}
\end{table}

\end{document}